\title{Max-plus statistical leverage scores}
\author{James Hook}

\documentclass[12pt]{article}
\usepackage{eqnarray,amssymb}
\usepackage{amsmath,amsthm}
\usepackage{pgf}
\usepackage{tikz}
\usetikzlibrary{arrows,automata}
\usepackage[latin1]{inputenc}
\usepackage{graphicx}
\newtheorem{theorem}{Theorem}[section]
\newtheorem{lemma}[theorem]{Lemma}

\usepackage{float}
\usepackage{enumitem}
\usepackage{subfigure}
\usepackage{caption}
\usepackage{longtable}

\DeclareFontFamily{OT1}{pzc}{}
\DeclareFontShape{OT1}{pzc}{m}{it}{<-> s * [1.10] pzcmi7t}{}
\DeclareMathAlphabet{\mathpzc}{OT1}{pzc}{m}{it}

\usepackage{algorithm}
\usepackage{algpseudocode}

\DeclareFontFamily{OT1}{pzc}{}
\DeclareFontShape{OT1}{pzc}{m}{it}{<-> s * [1.10] pzcmi7t}{}
\DeclareMathAlphabet{\mathpzc}{OT1}{pzc}{m}{it}
\def\lc#1{\mathpzc{#1}}
\def\uc#1{\mathcal{#1}}

\def\Rmax{\mathbb{R}_{\max}}
\def\R{\mathbb{R}}

\def\C{\mathbb{C}}

\def\P{\mathbb{C}\{\{z\}\}}
\def\Cnd{\mathbb{C}^{n\times d}}
\def\perm{\hbox{\normalfont perm}}
\def\oas{\hbox{\normalfont oas}}
\def\col{\hbox{\normalfont col}}

\def\uc{\mathcal}
\def\lc{\mathpzc}

\def\nbd{^{n\times d}}
\def\sgn{\hbox{sign}}

\newtheorem{heuristic}[theorem]{Heuristic}

\theoremstyle{definition}
\newtheorem{example}{Example}[section]

\usepackage[hmargin=2cm,vmargin=2cm]{geometry}

\begin{document}
\maketitle
\begin{abstract}
The statistical leverage scores of a complex matrix $A\in\Cnd$ record the degree of alignment between $\col(A)$ and the coordinate axes in $\C^n$.  These score are used in random sampling algorithms for solving certain numerical linear algebra problems. In this paper we present a max-plus algebraic analogue for statistical leverage scores. We show that max-plus statistical leverage scores can be used to calculate the exact asymptotic behavior of the conventional statistical leverage  scores of a generic matrices of Puiseux series and also provide a novel way to approximate the conventional statistical leverage scores of a fixed or complex matrix. The advantage of approximating a complex matrices scores with max-plus scores is that the max-plus scores can be computed very quickly. This approximation is typically accurate to within an order or magnitude and should be useful in practical problems where the true scores are known to vary widely. 

\end{abstract}

\section{Introduction}

Let $A\in\C^{n\times d}$ be a complex matrix. The \emph{statistical leverage scores} of $A$ are the vector $p(A)\in\R^n$, with
\begin{equation}\label{slsdef1}
p_{i}(A)=\left(\max_{x\in\C^n}\frac{|(Ax)_{i}|}{\|Ax\|_{2} }\right)^2, \quad \hbox{for $i=1,\dots,n$}.
\end{equation}
The $i$th statistical leverage score of $A$ is equal to the square of the cosine of the angle between $\col(A)$  and the unit vector $\underline{e}_{i}$. To calculate $p(A)$ we take a decomposition that provides an orthogonal basis for $\col(A)$. For example suppose that $A$ has rank $k$, then if we take the QR decomposition we obtain $A=QR$, with $Q\in\C^{n\times k}$ and
\begin{equation}\label{calcqr}
p_{i}(A)=\|Q_{i\cdot}\|_{2}^{2}, \quad \hbox{for $i=1,\dots,n$},
\end{equation}
where $Q_{i\cdot}$ denotes the $i$th row of $Q$. Note that $\sum_{i=1}^{n}p_{i}(A)/k=1$, so that the vector $p(A)/k\in\mathbb{R}^{n}$ is a probability distribution on $\{1,\dots,n\}$. 

Statistical leverage score distributions are used in random sampling algorithms for solving certain numerical linear algebra problems \cite{Woodruff2014,Ma2015,gittens2013}. For example: Algorithm~\ref{alg:SLS} approximates the least squares solution $x^{\ast}=\arg\min_{x\in\mathbb{C}^{d}}\| Ax-y\|_{2}$, by examining a randomly selected sample of the rows of $[A,y]$. The $r\times n$ random matrix $M$ samples $r$ rows from the least squares problem with replacement, according to the probability distribution $p$, the sampled rows are scaled by one over the square root of their sampling probability to ensure that the approximate solution is unbiased. We then compute the solution $\hat{x}$ that is optimal for the sampled rows and then use it as an approximate solution for the full problem.  This is similar to taking a poll to predict an election result and just like taking a poll, it is crucial that our sample set reflects the statistical properties of the full set of rows. Theorem~\ref{least2} shows that sampling with respect to statistical leverage scores is one way of achieving this. The full result presented in \cite{Drineas2006} also shows how approximate statistical leverage scores can also be used for sampling. Note that Theorem~\ref{least2} holds for an arbitrary matrix $A\in\mathbb{C}^{n\times d}$, in particular there is no assumed statistical model for the rows of $A$.

\begin{algorithm}

\caption{ \label{alg:SLS}
Given $A\in\C^{n\times d}$, $y\in\mathbb{C}^{n}$, a probability distribution $p\in\mathbb{R}^{n}$ and $r\in\{d,\dots,n\}$, compute $\hat{x}\approx x^{\ast}=\arg\min_{x\in\mathbb{C}^{d}}\| Ax-y\|_{2}$. }

\medskip

\begin{algorithmic}[1]

\For {$i=1,\dots,r$ independently} 
\State set $M_{i\cdot}=\underline{e}^{T}_{j}/\sqrt{p_{j}}$, with proability $p_{j}$, for $j=1,\dots,n$
\EndFor
\State set $\hat{x}=\arg\min_{x\in\mathbb{R}^{d}}\| (MA)x-My\|_{2}$

\end{algorithmic}

\end{algorithm}

\begin{theorem}[\cite{Drineas2006}, Theorem 3.1]\label{least2}

Let $A\in\mathbb{C}^{n\times d}$, $y\in\mathbb{C}^{n}$ and let $x^{\ast}=\arg\min_{x\in\mathbb{C}^{d}}\| Ax-y\|_{2}$, let $p\in\mathbb{R}^{n}$ be the statistical leverage scores of $[A,y]$ and let $\hat{x}$ be the ouput of Algorithm~\ref{alg:SLS}, then for
$$
r\geq 8 \frac{d+1}{\epsilon^2}\ln \frac{(d+1)}{\delta},
$$
we have
$$
\mathbb{P}\left\{ \| A\hat{x}-y\|_{2}\leq (1+2\epsilon)\| Ax^{\ast}-y\|_{2}\right\}>1-3\delta.
$$

\end{theorem}

In practice if we are considering solving an $n\times d$ least squares problem using a random sampling method, then we must be in a scenario where $\uc{O}(nd^2)$ computations are too costly and we are therefore unable to use \eqref{calcqr} to calculate the statistical leverage scores of the matrix $[A,y]$. Thus there is interest in developing efficient methods for approximating the statistical leverage scores of a matrix. Drineas et al present such a method in~\cite{Drineas2012}. Their approach uses random projections and can be tuned to provide approximations with a desired accuracy for a desired reliability probability. While the exact cost of computing this approximation depends on the chosen accuracy and probability it is $\uc{O}\big(nd\log(n)\big)$ for moderate values. In practice their method reliably produces approximate scores that are close enough to the true statistical leverage scores to be useful in sampling methods. 

In this paper we present max-plus statistical leverage scores, which provide a completely different method for approximating the statistical leverage scores of a complex matrix. Our method is potentially less reliable but faster than the method of Drineas et al.  Algorithm~\ref{alg:MPSLS}, presented in Section~\ref{asection}, returns the max-plus statistical leverage scores of a dense $n\times d$ matrix with average case cost $\uc{O}(nd+d^3)$, it also has huge potential to exploit sparsity and parallelism. We do not have a theorem guarantying that our approximation will always be within a certain accuracy. The max-plus approximation is a heuristic, which can go badly wrong for specially constructed `nasty' problems. In Section~\ref{nsection} we demonstrate that max-plus statistical leverage scores are able to provide an order of magnitude approximation to the scores of a range randomly generated numerical example problems. In Section \ref{psection} we prove that max-plus statistical leverage scores give the exact asymptotic behavior of the scores of generic matrices of Puiseux series. This theoretical connection between the max-plus and Puiseux series scores underpins Algorithm~\ref{alg:MPSLS} and provides some intuitive justification for the max-plus approximation of fixed complex matrices. 


\subsection{Quick introduction to max-plus algebra}

Max-plus algebra concerns the max-plus semiring $\Rmax=(\R\cup\{-\infty\},\oplus,\otimes)$ where
\begin{equation}
\lc{a}\oplus \lc{b}=\max\{a,b\}, \quad \lc{a}\otimes \lc{b}=\lc{a}+\lc{b}, \quad \hbox{for all $\lc{a,b}\in\Rmax$}.
\end{equation}
 Akian, Bapat and Gaubert show that max-plus algebra can be used
to calculate the exact asymptotic growth rates of the eigenvalues of generic matrices whose entries are Puiseux series~\cite{akia01,akia04}. Gaubert and Sharify where the first to exploit this idea to develop max-plus algebriac methods for approximating the order of magnitude of the eigenvalues of a fixed complex matrix polynomial~\cite{Gaubert2009}.  This approach has since been adapted and expanded to approximate matrix singular values and LU factors~\cite{hook15,hoti16}. In this paper we extend the approach further to include statistical leverage scores. We introduce a definition for max-plus statistical leverage scores, which enables us to calculate the exact asymptotic growth rates of the statistical leverage scores of generic matrices of Puiseux series, and to approximate the statistical leverage scores of a fixed complex matrix. We provide all of the necessary background material in this section. For a more thorougher introduction to max-plus algebra see~\cite{butk10}.

A max-plus matrix $\uc{A}\in\Rmax\nbd$ is simply an $n\times d$ array of elements from $\Rmax$. Max-plus matrix multiplication is defined in analogy to the classical case, for $\uc{A}\in\Rmax\nbd$ and $\uc{B}\in\Rmax^{d\times m}$, the product $(\uc{A}\otimes \uc{B})\in\Rmax^{n\times m}$ is the max-plus matrix with
\begin{equation}
(\uc{A}\otimes \uc{B})_{ij}=\bigoplus_{k=1}^{d}\lc{a}_{ik}\otimes\lc{b}_{kj}=\max_{k=1}^{d}(\lc{a}_{ik}+\lc{b}_{kj}).
\end{equation}
For clarity we will often display equations using max-plus algbraic notation alongside equivalent expressions that only use standard notation. We will use the max-norm to measure the size of vectors $\lc{y}\in\Rmax^n$
\begin{equation}
\|\lc{y}\|_{\max}=\bigoplus_{i=1}^{n}\lc{y}_{i}= \max_{i=1}^{n}\lc{y}_{i},
\end{equation}
although clearly this is not a norm in the usual sense. Instead we can think of $\|\lc{y}\|_{\max}$ as the max-plus algebraic analogue of a $p$-norm over $\C^n$. If we take the formula for the $p$-norm and replace additions with maximums, multiplications with additions and powers with multiplication by scalars, then we arrive at the formula for the max-norm, except for the use of absolute values, which do not have a precise max-plus algebraic analogue.

For $\uc{A}\in\Rmax^{n\times d}$, we define the \emph{max-plus permanent} of $\uc{A}$ by
\begin{equation}\label{permdef}
\perm(\uc{A})=\bigoplus_{\phi\in\Phi(d,n)}\bigotimes_{j=1}^{d}\lc{a}_{\phi(j)j}=\max_{\phi\in\Phi(d,n)}\sum_{j=1}^{d}\lc{a}_{\phi(j)j}, 
\end{equation}
where $\Phi(d,n)$ be the set of all injections from $\{1,\dots,d\}$ to $\{1,\dots,n\}$.  We also define the set of \emph{optimal assignments} of $\uc{A}$ by
\begin{equation}\label{oasdef}
\oas(\uc{A})=\arg\max_{\phi\in\Phi(d,n)}\sum_{j=1}^{d}\lc{a}_{\phi(j)j}.
\end{equation}
For $\phi\in\oas(\uc{A})$ and $i\in\{1,\dots,n\}$, if $\phi(j)=i$ for some $j\in\{1,\dots,d\}$, then we say that $\phi$ \emph{assigns} row $i$ to column $j$ and vice versa. Note that for a square matrix $\uc{A}\in\Rmax^{n\times n}$, the formula \eqref{oasdef} is analogous to the classical formula for the determinant, except for the alternating sign term, which has no max-plus algebraic analogue since $\oplus$ is not invertible. We exploit this connection between max-plus permanents and conventional determinants extensively in Section 3.

Optimal assignments also have a neat operational research interpretation. Suppose that we have $n$ workers and $d$ jobs and that we must assign each job to a unique worker. Let $\uc{A}\in\Rmax^{\nbd}$ be the max-plus matrix with $\lc{a}_{ij}$ equal to the benefit of assigning worker $i$ to job $j$. Then $\perm(\uc{A})$ is the maximum possible total benefit and $\oas(\uc{A})$ is the set of optimal assignments of jobs to workers. 

For  $\uc{A}\in\Rmax^{n\times d}$ and for $i=1,\dots,n$, we define the $i$-\emph{obligated permanent} of $\uc{A}$ by
\begin{equation}\label{opermdef}
\perm(\uc{A},i)=\bigoplus_{\phi\in\Phi(d,n;i)}\bigotimes_{j=1}^{d}\lc{a}_{\phi(j)j}=\max_{\phi\in\Phi(d,n;i)}\sum_{j=1}^{d}\lc{a}_{\phi(j)j}, 
\end{equation}
where $\Phi(d,n;i)$ is the set of all injections $\phi$ from $\{1,\dots,d\}$ to $\{1,\dots,n\}$, with $\phi(j)=i$, for some $j=1,\dots,d$. The $i$-obligated permanent is the maximum weight of an assignment that assignes row $i$. 

Throughout this paper complex matrices will be denoted by capital letters with
their entries denoted by the corresponding lower case letter in the usual way
$A=(a_{ij})\in\C\nbd$. Matrices of Puiseux series will be denoted by
capital letters with a tilde and their entries by the corresponding lower case
letter also with a tilde $\widetilde{A}=(\tilde{a}_{ij})\in\C\{\{z\}\}\nbd$,
where $\C\{\{z\}\}$ denotes the field of Puiseux series. Max-plus matrices will
be denoted by calligraphic capital letters and their entries by the
corresponding lower case calligraphic letter
$\mathcal{A}=(\mathpzc{a}_{ij})\in\Rmax\nbd$. 

For an $n\times d$ matrix $A$, we use the notation  $A([i_{1},\dots,i_m],[j_1,j\dots,j_k])$ to denote the $m\times k$ matrix formed from the $\{i_1,\dots,i_m\}$ rows and $\{j_1,\dots,j_k\}$ columns of $A$. We also use the notation $A([i_{1},\dots,i_m]^c,[j_1,j\dots,j_k]^c)$ to denote the $(n-m)\times (d-k)$ matrix formed from the $\{1,\dots,n\}/\{i_1,\dots,i_m\}$ rows and $\{1,\dots,d\}/\{j_1,\dots,j_k\}$ columns of $A$.

\subsection{Max-plus statistical leverage scores}\label{maxplussection}

Let $\uc{A}\in\Rmax^{n\times d}$ be a max-plus matrix. In analogy to \eqref{slsdef1} we define the \emph{naive max-plus statistical leverage scores}  of $\uc{A}$ to be the vector $\lc{q}(\uc{A})\in\Rmax^n$ with
\begin{align}\label{mpslsdef1}
\lc{q}_{i}(\uc{A}) & =\max_{\lc{x}\in\Rmax^d}\Big((\uc{A}\otimes \lc{x})_{i}-\|\uc{A}\otimes \lc{x}\|_{\max}\Big)^{\otimes 2} \\ &=2\max_{\lc{x}\in\Rmax^d}\big(\max_{j=1}^d (\lc{a}_{ij}+\lc{x}_{j})-\max_{i=1}^n \max_{k=1}^d (\lc{a}_{ik}+\lc{x}_{k})\Big), \quad \hbox{for $i=1,\dots,n$}. \nonumber
\end{align}
Although this definition of max-plus statistical leverage scores looks very similar to the complex matrix case, it turns out to not be very useful for our purposes. Instead we will use a slightly different definition, which is chosen to permit the results presented in Section \ref{psection}. 

We define the \emph{max-plus statistical leverage score} of $\uc{A}$ to be the vector $\lc{p}(\uc{A})\in\Rmax^n$ with
\begin{equation}\label{mpslsdef2}
\lc{p}_{i}\big(\uc{A})=2\big(\perm(\uc{A},i)-\perm(\uc{A})\big), \quad \hbox{for $i=1,\dots,n$}.
\end{equation}
If there is an optimal assignment of $\uc{A}$ that assigns row $i$, then $\lc{p}_{i}(\uc{A})=0$. Otherwise, $\lc{p}_{i}\big(\uc{A})$ equals minus two times the smallest bonuses that needs to be applied to row $i$ in order for there to exist an optimal assignment that assigns row $i$.

We can convert any vector $\lc{p}\in\Rmax^n$ into a probability distribution using the \emph{softmax} function 
\begin{equation}\label{softmax}
\sigma(\lc{p})_{i}=\frac{10^{\lc{p}_{i}}}{\big(\sum_{j=1}^{n}10^{\lc{p}_{j}}\big)}, \quad \hbox{for $i=1,\dots,n$}.
\end{equation}
The following heuristic shows how we can use max-plus statistical leverage scores to approximate the conventional statistical leverage scores of a complex matrix. 
\begin{heuristic}\label{happrox}
Let $A\in\C\nbd$ be of rank $k$, then
$$
p(A)/k\approx \sigma\big(\lc{p}(\log |A|)\big),
$$
where $\log |A|\in\Rmax\nbd$ is the componentwise log of absolute value of $A$.
\end{heuristic}

\begin{example}\label{goodegg1}
Consider
$$
A=\left[\begin{array}{cc} 1000 & 1000 \\ 1 & 100 \\ 10 & 1 \end{array}\right], \quad \uc{A}=\log|A|=\left[\begin{array}{cc} 3 & 3 \\ 0 & 2 \\ 1 & 0 \end{array}\right]. 
$$
By taking the QR decomposition of $A$ we calculate its statistical leverage score distribution to be $p(A)=[0.4999  ,  0.4959  ,  0.0041]$. The naive max-plus statistical leverage scores of $\uc{A}$ are given by $\lc{q}(\uc{A})=[0,-1,-2]$. To compute the max-plus statistical leverage scores of $\uc{A}$ we need to compute the permanent of $\uc{A}$. This is given by $\perm(\uc{A})=3+2=5$, which is attained by $\phi=(1,2)$. Since $\phi$ assigns rows 1 and 2, these rows have score zero. The $3$-obligated permanent is given by $\perm(\uc{A},3)=3+1=4$, which is attained by $\phi_{3}=(3,1)$. The statistical leverage scores are therefore given by $\lc{p}(\uc{A})=[0,0,-2]$.

The naive max-plus statistical leverage scores of $\uc{A}$ result in the distribution $\sigma\big(\lc{q}(\uc{A})\big)=[ 0.9009,    0.0901   , 0.0090]$. The max-plus statistical leverage scores of $\uc{A}$ result in the distribution $\sigma\big(\lc{p}(\uc{A})\big)=[0.4975 ,   0.4975,    0.0050]$, which provides an order of magnitude approximation to the true statistical leverage score distribution $p(A)$.

\end{example}

\begin{example}\label{badegg1}
Consider
$$
A=\left[\begin{array}{cc} 1000 & 1000 \\ 1 & 100 \\ 10 & 10 \end{array}\right], \quad \uc{A}=\log|A|=\left[\begin{array}{cc} 3 & 3 \\ 0 & 2 \\ 1 & 1 \end{array}\right]. 
$$
We have $p(A)=[0.5000,    0.5000,    0.00005]$, but $\lc{p}(\uc{A})=[0,0,-2]$, which results in the approximate scores $\sigma\big(\lc{p}(\uc{A})\big)=[0.4975 ,   0.4975,    0.0050]$. In this example the max-plus approximation fails to capture the order of magnitude of the score of row 3.
\end{example}

To understand why the max-plus approximation works better for Example \ref{goodegg1} than Example \ref{badegg1} it is important to consider the fact that the max-plus approximation is `blind to sign' in the sense that it does not depend on the sign or complex argument of the entries in $A$. To illustrate this point further we randomly generate many matrices with the same sized entires as the previous example problems but with independent, uniformly distributed complex arguments. For each of these randomly generated matrices we compute the log of the statistical leverage score of row 3 and plot a histogram of the results. See Figure \ref{hisplot}. Note that for Example  \ref{goodegg1} the scores are always confined to a narrow band, which is within an order of magnitude of the max-plus approximation. But that for Example \ref{badegg1} the scores have a light lower tail, so that there is a set of small measure for which the max-plus approximation is not accurate to within an order of magnitude. The matrix $A$ in Example \ref{badegg1} belongs to this small measure set.

The only direct support for Heuristic~\ref{happrox} comes from the empirical evidence presented in Section~\ref{nsection}. We have no theorem saying that the accuracy of the approximation should be within a certain factor for an arbitrary complex matrix. However, for nearly all of the test matrices $A$ that we have examined, we find that the max-plus approximation does capture the order of magnitude of all of $A$'s statistical leverage scores. In the few cases where the statistical leverage scores of the matrix $A$ are poorly approximated by the max-plus approximation, we find that applying a random perturbation to the complex arguments of the entries of $A$ results in a matrix $A'$ whose statistical leverage scores are well approximated by the max-plus approximation. In this sense we say that the max-plus approximation provides an order of magnitude approximation for the statistical leverage scores for all but a small measure set of `nasty' matrices. 

Of course this is cold comfort if we are interested in approximating the scores of a particular matrix $A$ that happens to fall in this nasty set. Ultimately the success of the max-plus approximation in practice will depend on identifying domains of problems and compatible preprocessing techniques that give rise to matrices where these problems either tend not to occur, or where some of these problems occur but the max-plus approximate statistical leverage scores are still useful in downstream applications. We should note that issues of this sort are common to other methods which use max-plus algebra to approximate classical linear algebra objects including eigenvalues~\cite{Gaubert2009} and LU factors~\cite{hoti16}. In these other applications we find that max-plus methods tend to work well on large sparse matrices from practical problems, particularly highly unstructured problems with a large range of entry sizes.

 \begin{figure}
\begin{center}
\subfigure[]{\includegraphics[scale=0.4]{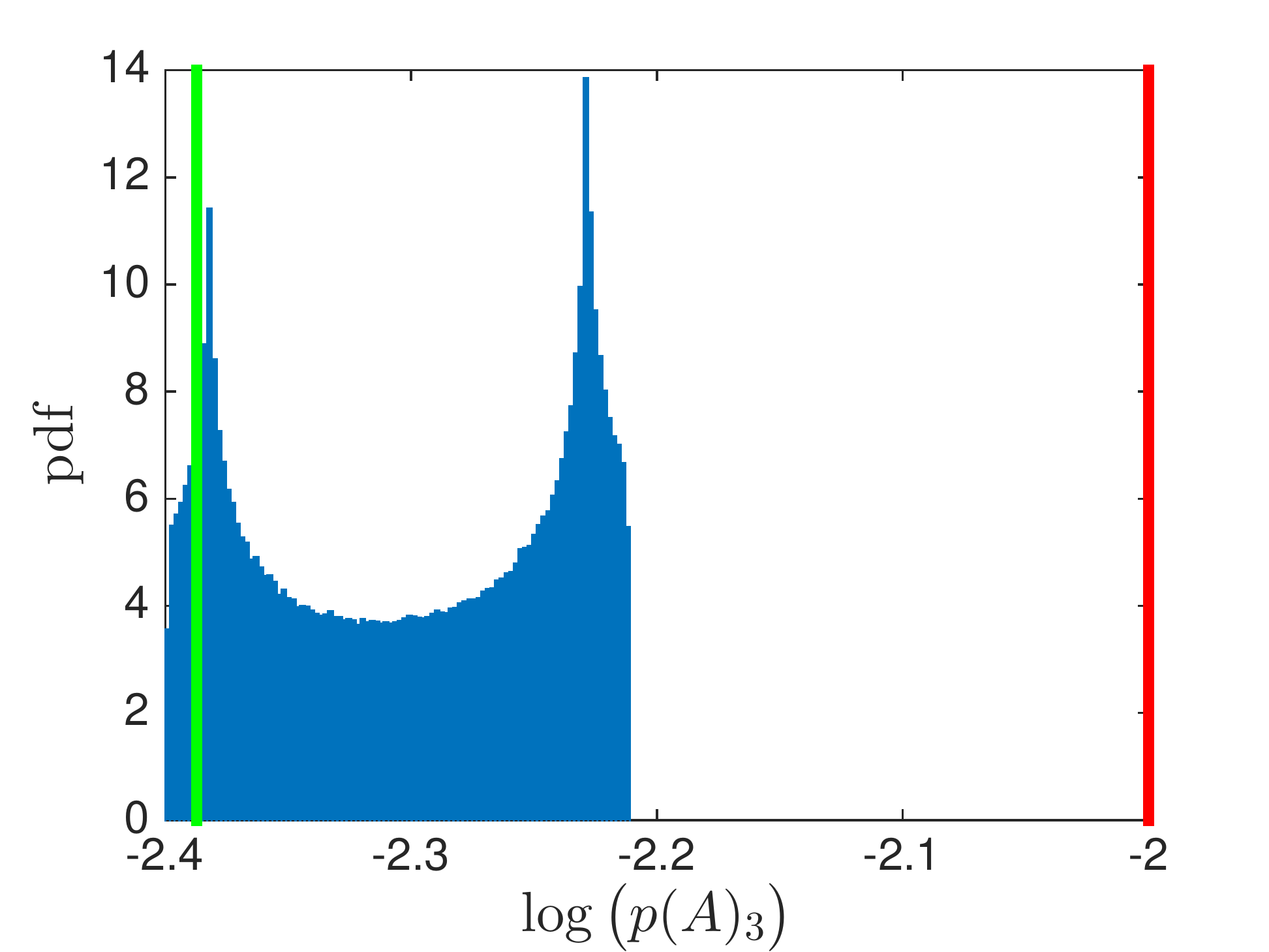}}\hspace{.4cm}
\subfigure[]{\includegraphics[scale=0.4]{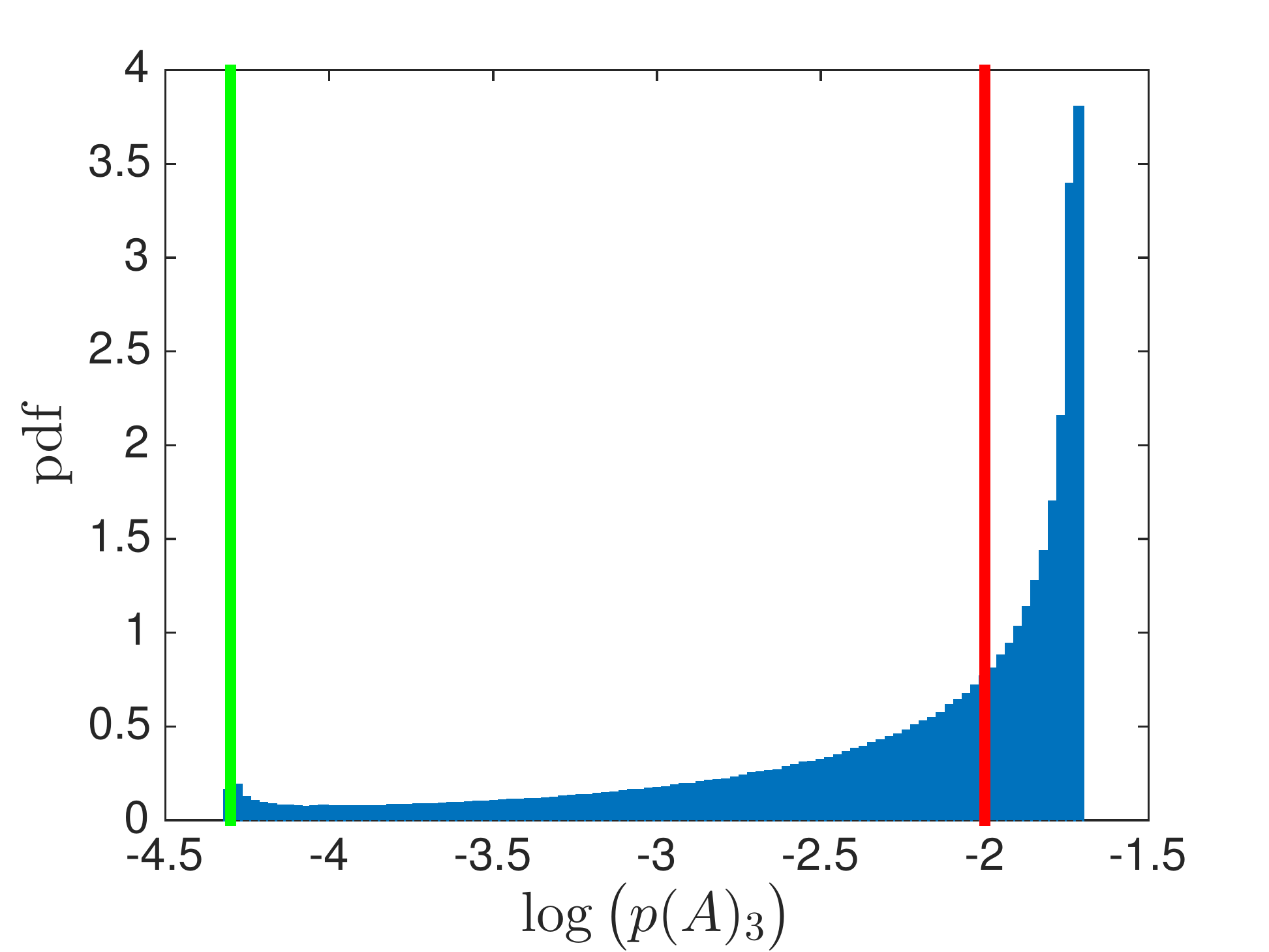}}\hspace{.4cm}
\caption{Distribution of $\log\big(p(A)_{3}\big)$ for randomly generated matrices based on a) Example \ref{goodegg1} , b) Example \ref{badegg1}. Max-plus approximation marked in red, score of example problem in green.}
\label{hisplot}
\end{center}
\end{figure}

\section{Numerical experiments}\label{nsection}

In this section we apply the max-plus approximation of Heuristic~\ref{happrox} to some larger numerical examples and compare using different sampling distributions in Algorithm~\ref{alg:SLS}. The matrices that we use are randomly generated using the scheme set out in (\cite{Ma2015}, Section 4.1). The \emph{coherence} of a matrix $A\in\Cnd$ is equal to its largest individual statistical leverage score. A coherent matrix, with a large coherence value, will have a wide range of statistical leverage scores. An incoherent matrix, with a small coherence value, will have more uniform statistical leverage scores. We set $n=10^5$, $d=50+1$ and $\Sigma\in\mathbb{R}^{d\times d}$, with $\Sigma_{ij}=2\times 0.5^{|i-j|}$. The example matrices are generated as follows.

\smallskip

\indent {\bf Incoherent example:} Each row of $A\in\mathbb{R}^{n\times d}$ is chosen independently from a multivariate Gaussian distribution $\uc{N}(\underline{1},\Sigma)$, where $\underline{1}\in\mathbb{R}^{d}$ is a vector of ones. 

\smallskip

\indent {\bf Semi-coherent example:} Each row of $A\in\mathbb{R}^{n\times d}$ is chosen independently from a multivariate t-distribution $t_{3}(\underline{1},\Sigma)$, with three degrees of freedom.

\smallskip

\indent {\bf Coherent example:} Each row of $A\in\mathbb{R}^{n\times d}$ is chosen independently from a multivariate t-distribution $t_{1}(\underline{1},\Sigma)$, with one degree of freedom.

\smallskip

For each matrix  $A\in\mathbb{R}^{n\times d}$, we compute the exact statistical leverage score probability distribution $p(A)/d$, using \eqref{calcqr}. We use  Algorithm~\ref{alg:MPSLS} to compute the max-plus approximation $\sigma\big(\lc{p}(\log|A|)\big)$. For comparison we also compute an alternative statistical leverage score approximation which, like the max-plus approximation, only depends on the size of the entries in $A$. The \emph{column normalized row norms} (CNRN) scores of $A$ are given by
\begin{equation}\label{cnrn}
q(A)_{i}=\| C_{i\cdot}\|_{2}^{2}, \quad \hbox{for $i=1,\dots,n$},
\end{equation}
where $C=AD^{-1}$ and $D\in\mathbb{R}^{d\times d}$ is the diagonal matrix with $d_{jj}=\| A_{\cdot j}\|_{2}$, for $j=1,\dots,d$. We compute the CNRN probability distribution $q(A)/d$ for each matrix. Note that, although their computation is far more straightforwards, the cost of computing the CNRB scores is the same order as the max-plus scores.

For each example matrix we formulate and solve the least squares problem $x^{\ast}=\arg\min_{x\in\mathbb{R}^{d-1}}\|Bx-y\|_{2}$, where $B=A([1,\dots,n],[1,\dots,d-1])$ and $y=A([1,\dots,n],[d])$. We then compute approximate solutions using Algorithm~\ref{alg:SLS}. For each different sampling distribution and a range of values of $r$, we run one hundred independent instances of Algorithm~\ref{alg:SLS}. The results of these experiments are displayed in Figure~\ref{numplot}.

For the incoherent example the exact statistical leverage scores are nearly uniform. Both approximation methods capture the order of magnitude of all of the scores and all of the different sampling methods have the same performance. For the semi-coherent example the exact statistical leverage scores range between $10^{-2}$ and $10^{-6}$. Both approximation methods capture the order of magnitude of all of the scores, but the CNRN approximation is slightly more accurate. The rows with the largest exact scores are under approximated by the max-plus scores, but never by more than a factor of ten. The uniform sampling method does not perform as well as the other methods in this example. For the coherent example the exact statistical leverage scores range between $10^{-2}$ and $10^{-10}$. The max-plus approximation captures the order of magnitude of all of the scores but the CNRN scores over approximate the largest scores and under approximate many of the smaller scores. The uniform sampling method performs very poorly on this problem, the CNRN method does not perform as well as the exact statistical leverage scores method or the max-plus approximation method, which both perform well.

 \begin{figure}
\begin{center}
\subfigure[]{\includegraphics[scale=0.3]{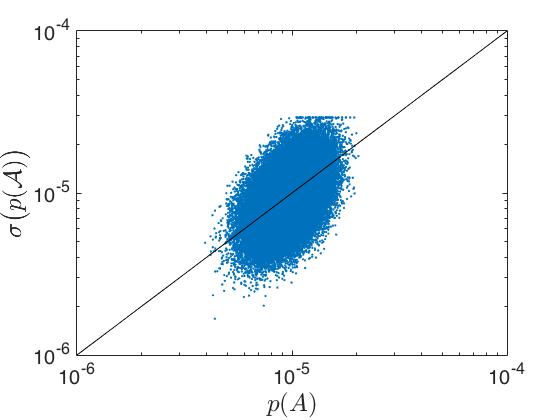}\includegraphics[scale=0.3]{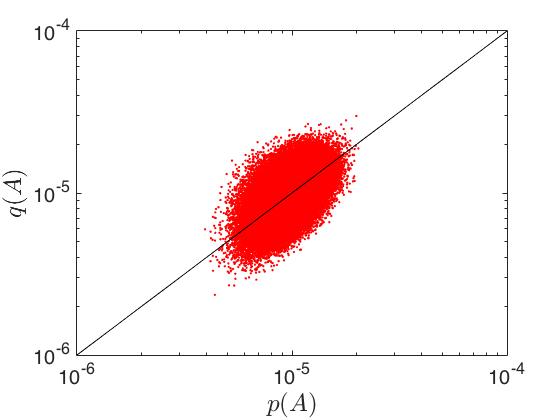}\includegraphics[scale=0.3]{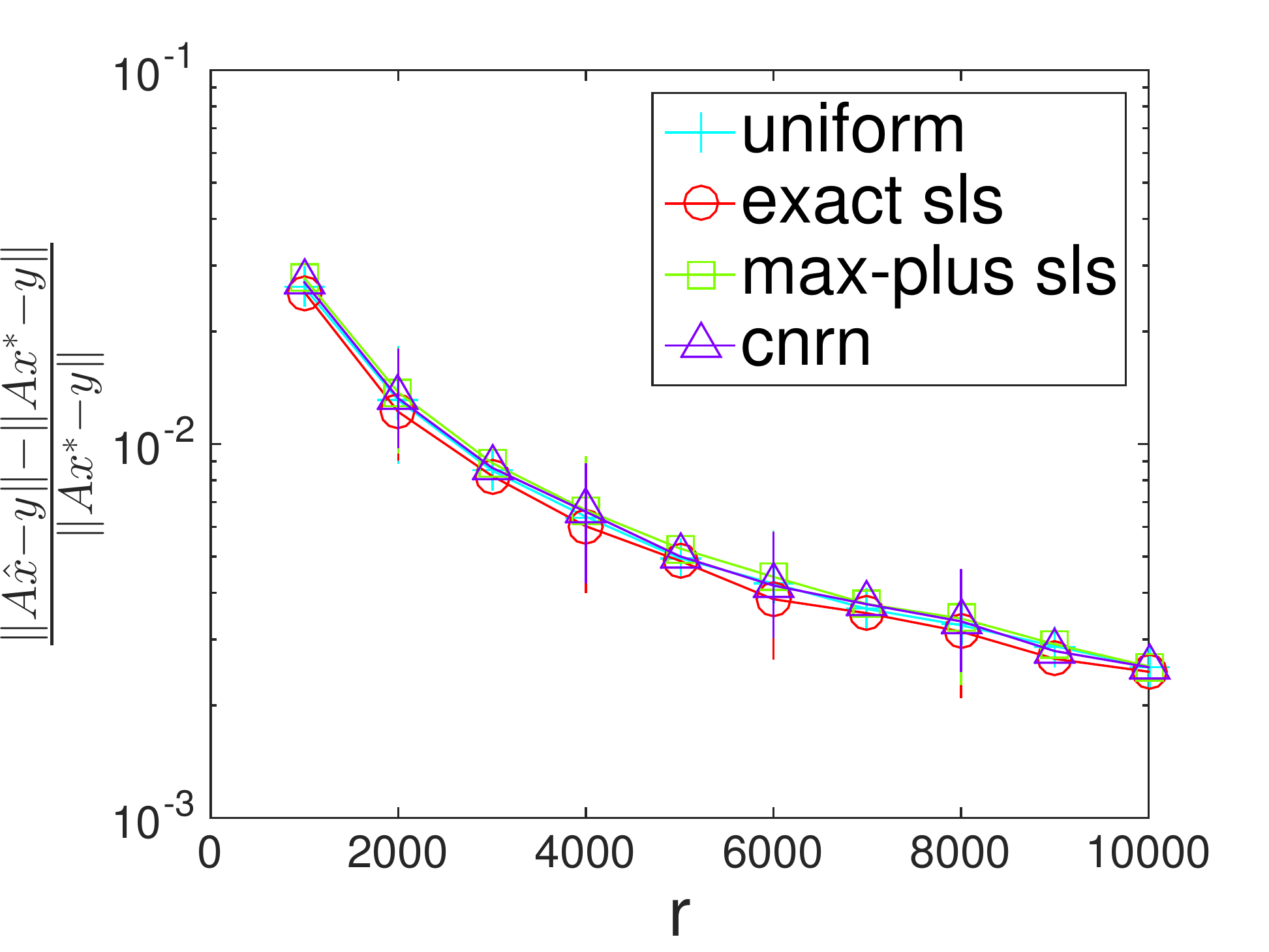}}

\subfigure[]{\includegraphics[scale=0.3]{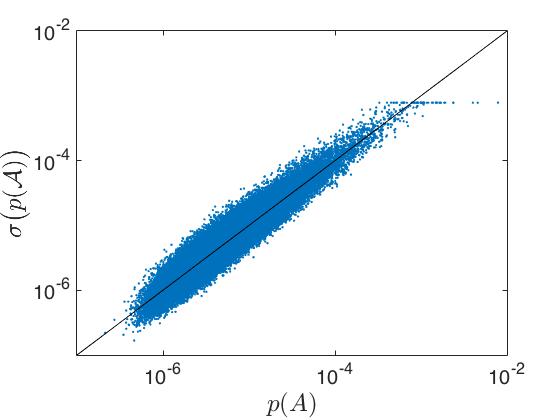}\includegraphics[scale=0.3]{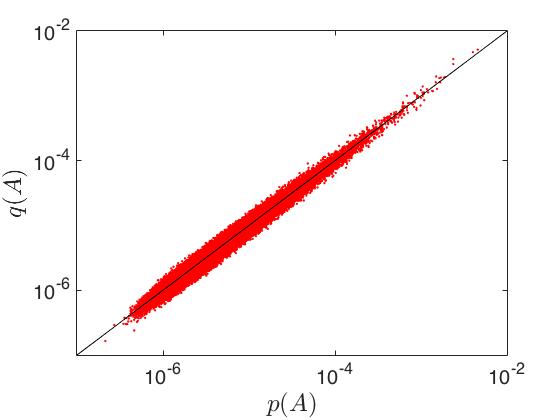}\includegraphics[scale=0.3]{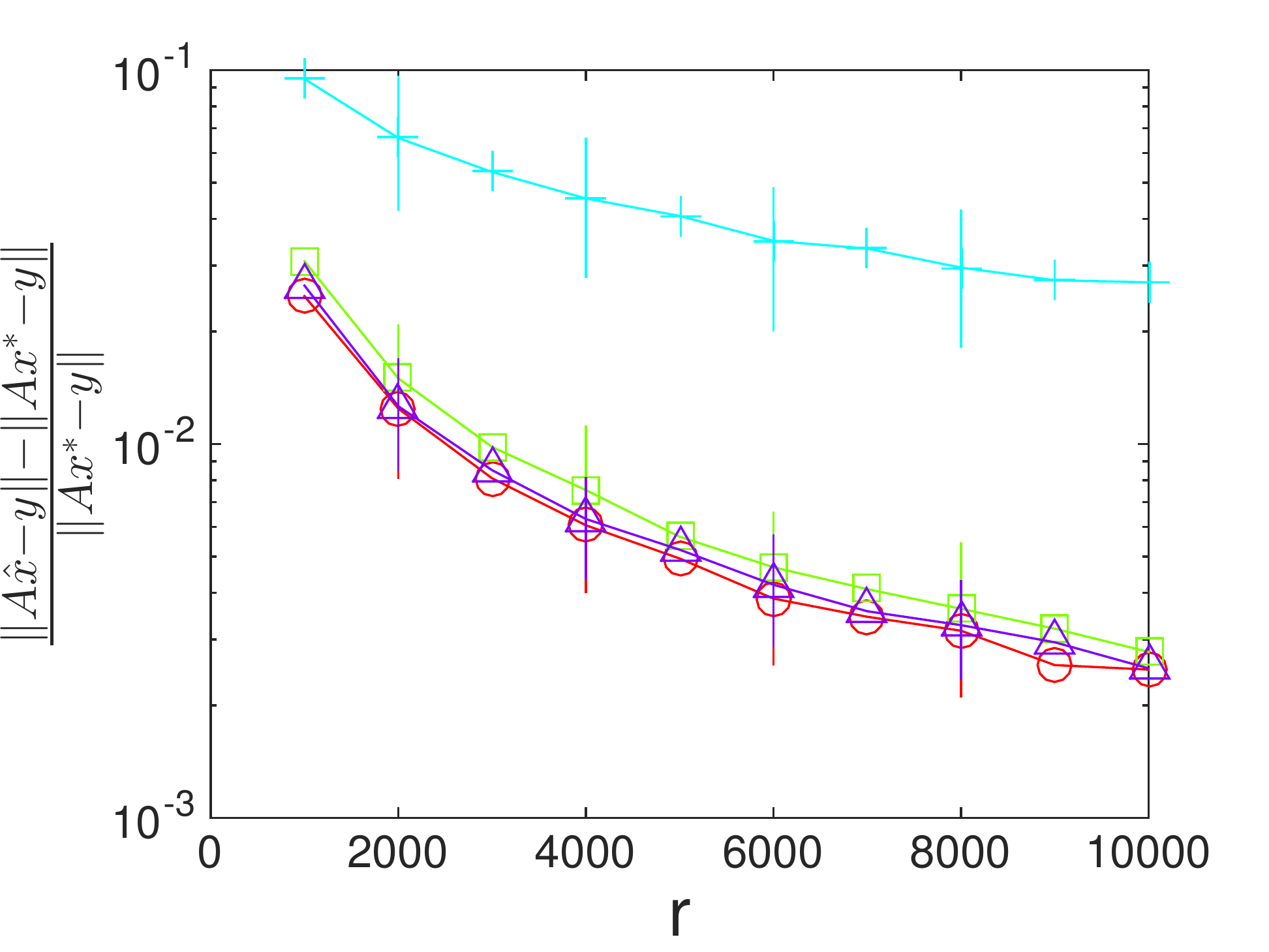}}

\subfigure[]{\includegraphics[scale=0.3]{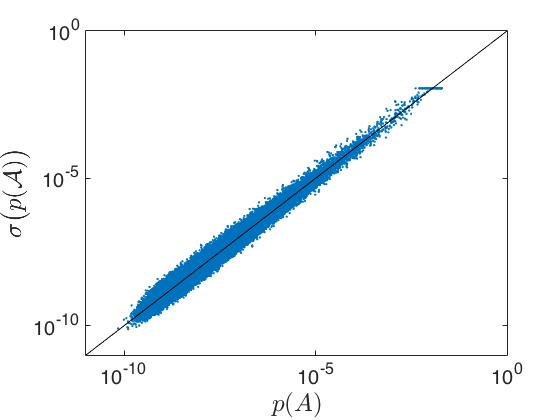}\includegraphics[scale=0.3]{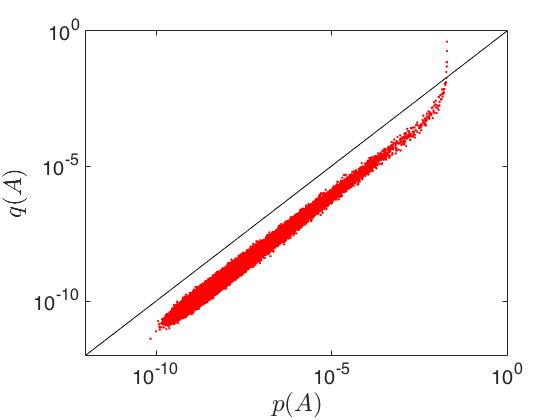}\includegraphics[scale=0.3]{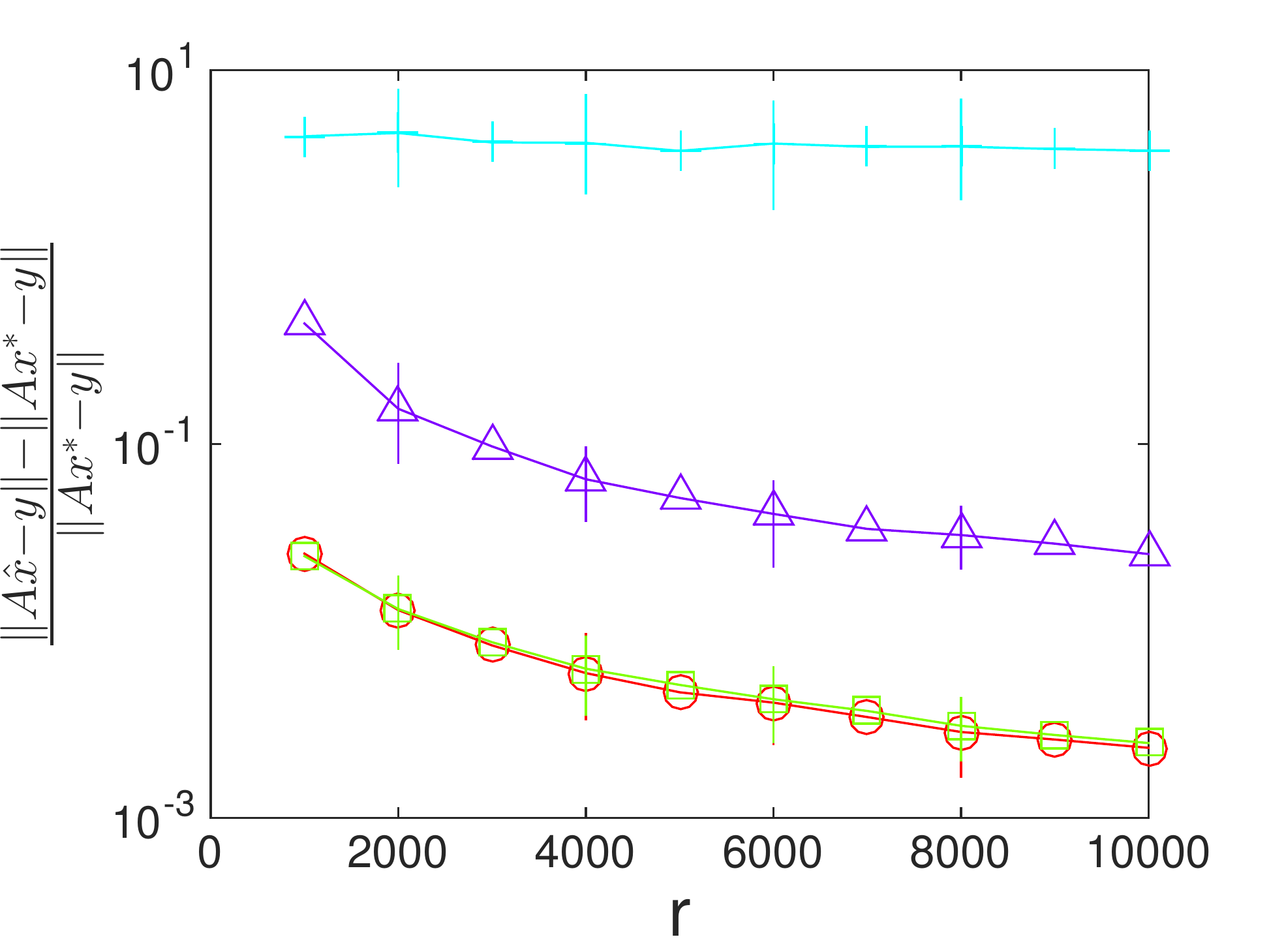}}
\caption{Left: scatterplot of exact statistical leverage scores vs max-plus approximation. Middle: scatter plot of exact statistical leverage scores vs CNRN approximation. Right: error vs sample size for sampled least squares approximations. Errors plotted are the geometric mean of one hundred independent trials, vertical bars show 90\% range. For a) incoherent example, b) semi-coherent example, c) coherent example.   }
\label{numplot}
\end{center}
\end{figure}

\section{Asymptotics of generic matrices of Puiseux series}\label{psection}

Whilst we cannot prove any results relating the statistical leverage scores of a complex matrix to the max-plus approximation we are able to prove a connection between the statistical leverage scores of a matrix of Puiseux series and the max-plus statistical leverage scores of a related matrix. Although we are ultimately interested in fixed complex matrices and not matrices of Puiseux series, this theory is important as it provides the intuition as to how and why the max-plus approximation is able to work. Indeed, the proof of  Theorem \ref{lemthm} can be viewed as a derivation for our method and understanding its proof will give useful insight into the workings of Algorithm~\ref{alg:MPSLS}.

A Puiseux series is a generalization of a power series that allows for negative and fractional powers:
\begin{equation}
\tilde{f}(z)=\sum_{i=k}^{\infty}c_{i}z^{\frac{i}{m}},
\end{equation}
for $m\in\mathbb{N}$, $k\in\mathbb{Z}$, $c_{i}\in\C$ for $i\ge k$ and $c_{k}\neq
0$.  Puiseux series form an algebraically closed field under addition and multiplication which we denote
$\mathbb{C}\{\{z\}\}$. 

We define the \emph{coefficient of the lowest order term} map $L:\P\mapsto \C$, by $L(\tilde{f})=c_{k}$. We also define the \emph{valuation} map $V:\P\mapsto \Rmax$, by $V(\tilde{f})=\frac{-k}{n}$, i.e. minus the degree of the lowest order term. The valuation of a Puiseux series tells us how quickly it blows up when evaluated at small values of $z$. Valuation provides an almost everywhere homeomorphism since
\begin{align}
\label{timesplus} V(\tilde{f}\tilde{g})&=V(\tilde{f})\otimes V(\tilde{g}), \quad \hbox{for all $\tilde{f},\tilde{g}\in\P$}, \\
\nonumber V(\tilde{f}+\tilde{g})&=V(\tilde{f})\oplus V(\tilde{g}), \quad \hbox{for almost all $\tilde{f},\tilde{g}\in\P$},
\end{align}
where the second property holds unless $V(\tilde{f})=V(\tilde{g})$ and $L(\tilde{f})=-L(\tilde{g})$. We apply $V$ and $L$ to matrices of Puiseux series componentwise in the obvious way.

Since there is no Puiseux series analogue for the complex conjugate,  we are not able to define a consistent inner product on $\P^n$ and therefore cannot define the statistical leverage scores of a matrix of Puiseux series directly. Instead we will evaluate our matrix of Puiseux series $\tilde{A}\in\P\nbd$, at a value $z\in\C$, to obtain a complex matrix $\tilde{A}(z)\in\C\nbd$ and then compute the statistical leverage scores of this matrix $p\big(\tilde{A}(z)\big)$. These scores can then be thought of as a function of $z$. Although $p\big(\tilde{A}(z)\big)$ will typically not be a Puiseux series, we can still measure its asymptotic growth rate for small $z$ to give the same characterization as valuation. Theorem \ref{lemthm}, which is the main result of this section, states that the asymptotic growth rates of the statistical leverage scores of $\tilde{A}(z)$ are equal to the max-plus statistical leverage scores of the valuation $V(\tilde{A})$.

In the remainder of this section, whenever we are working with a matrix of Puiseux series $\tilde{A}\in\P\nbd$, we will assume that $\tilde{A}$ has no entries identically equal to zero. The same results can be obtained for the case where $\tilde{A}$ contains entries equal to zero but the technical details of the proofs become a lot more complicated.

For $n,d\in\mathbb{N}$ with $n\geq d$, define the set of \emph{good coefficient matrices} $\uc{G}(n,d)\subset \C\nbd$ by 
\begin{equation}
\uc{G}(n,d)=\{C\in\C\nbd: \sum_{\pi\in\Psi}\sgn(\pi)\prod_{k=1}^{d}c_{\phi(k)\pi(k)}\neq 0, \hbox{ for all $\Psi\subset \Pi(d)$ and $\phi\in\Phi(d,n)$}\},
\end{equation}
where $\Pi(d)$ is the set of permutations of $\{1,\dots,d\}$ and $\Phi(d,n)$ is the set of injections from $\{1,\dots,d\}$ to $\{1,\dots,n\}$. In what follows, for a matrix $\tilde{A}\in\P\nbd$, the condition $L(\tilde{A})\in\uc{G}(n,d)$, will be a sufficient condition for our results to hold. The following lemma justifies us saying that our results are therefore generically true for matrices of Puiseux series. Note that although the results that appear in \cite{hoti16} are for square matrices,  generalizing them to rectangular matrices is straightforwards. 

\begin{lemma}[\cite{hoti16}, Lemma 4.2]\label{generic}
$\uc{G}(n,d)$ is a generic (open and dense) subset of $\C\nbd$. Therefore $\{\uc{A}\in\P\nbd:L(\uc{A})\in\uc{G}(n,d)\}$ is a  generic subset of $\P\nbd$, with respect to the topology induced by $L$. 
\end{lemma}

\begin{lemma}[\cite{hoti16}, Lemma 4.4]\label{submatrix}
Let $C\in\uc{G}(n,d)$ and let $C([i_{1},\dots,i_{m}],[j_{1},\dots,j_{\ell}])$ be an $m\times \ell$ submatrix of $C$, then $C([i_{1},\dots,i_{m}],[j_{1},\dots,j_{\ell}])\in\uc{G}(m,\ell)$.
\end{lemma}

The following Lemma is fundamental to our approach of using max-plus algebra to calculate the asymptotic behavior of generic matrices of Puiseux series. 
\begin{lemma}[\cite{hoti16}, Lemma 4.1]\label{detperm}
Let $\tilde{M}\in\P^{d\times d}$ and suppose that $\uc{L}(\tilde{M})\in\uc{G}(d,d)$, then
$$
V\big(\det(\tilde{M})\big)= \perm\big(V(\tilde{M})\big).
$$
\end{lemma}
\begin{example}
Consider
$$
\tilde{M}=\left[\begin{array}{cc} z^{-3}+2z & z^{-2} \\ -z^{-2}+2 &  1+z \end{array}\right], \quad V(\tilde{M})=\left[\begin{array}{cc} 3 & 2  \\ 2 &  0 \end{array}\right].
$$
We have $\det(\tilde{M})=z^{-4}+z^{-3}-z^{-2}+2z+2z^2$ and $V\big(\det(\tilde{M})\big)=4$, which agrees with $\perm\big(V(\tilde{M})\big)=\max\{3+0,2+2\}=4$.
\end{example}

Max-plus linear systems, i.e. equations of the form $\mathcal{A}\otimes
 x=b$, have many applications in scheduling and dynamical systems~\cite{butk10,heid2006}.  Such systems are better understood by studying the symmetrization of max-plus algebra $\mathbb{S}$, which is an extension of
 $\mathbb{R}_{\max}$, that allows for a kind of max subtraction operation (See
 \cite{Baccelli01}, section 3.4 for an introduction).  In this setting it is possible
 to either solve or determine that no solution exists to certain max-plus linear
 equations using a max-plus analogue of Cramer's rule (\cite{Baccelli01}, Section 3.5.2). This approach uses an expression
 for the max-plus inverse of a max-plus matrix, which looks exactly like the conventional Cramer's rule inverse, only with permanents instead of determinants. Typically this max-plus inverse does not provide a functional inverse in the usual sense as only a small subset of all max-plus matrices are invertible. In Lemma~\ref{inverse} we show how to use this same max-plus inverse expression to calculate the asymptotic growth rates of the entries in the inverse of a matrix of Puiseux series. 
 
For $\uc{M}\in\Rmax^{d\times d}$, define the \emph{max-plus inverse} $\uc{M}^{\otimes -1}\in\Rmax^{d\times d}$ by
\begin{equation}\label{maxplusinv}
(\uc{M}^{\otimes -1})_{ij}=\perm\big(\uc{M}([j]^{c},[i]^{c})\big)-\perm(\uc{M}),\quad \hbox{for $i,j=1,\dots,d$}.
\end{equation}
\begin{lemma}\label{inverse}Let $\tilde{M}\in\P^{d\times d}$ and suppose that $\uc{L}(\tilde{A})\in\uc{G}(d,d)$, then $\tilde{M}$ is invertible and
$$
V(\tilde{M}^{-1}) =V(\tilde{M})^{\otimes -1}.
$$
\end{lemma}
\begin{proof}
From Cramer's rule we have
$$
(\tilde{M}^{-1})_{ij}=\det\big(\tilde{M}([j]^{c},[i]^{c})\big)/\det(\tilde{M}), \quad \hbox{for $i,j=1,\dots,d$}.
$$
Taking the valuation and using \eqref{timesplus} we have
$$
V(\tilde{M}^{-1})_{ij}=V\Big(\det\big(\tilde{M}([j]^{c},[i]^{c})\big)\Big)-V\big(\det(\tilde{M})\big), \quad \hbox{for $i,j=1,\dots,d$}.
$$
The result follows from a simple application of Lemma \ref{submatrix} and Lemma \ref{detperm}.
\end{proof}

The following Lemma is a technical result which we use in the proof of Theorem \ref{lemthm}.
\begin{lemma}\label{trickylemma}
Let $\uc{A}\in\Rmax\nbd$, let $\phi\in\oas(\uc{A})$ and let $j\in\{1,\dots,d\}$, then 
$$
\perm\big(\uc{A}([1,\dots,n],[j]^{c})\big)=\perm\big(\uc{A}([\phi(1),\dots,\phi(d)],[j]^{c})\big).
$$
\end{lemma}
\begin{proof}
First note that, since the LHS is the maximum over a set of assignments that includes all of the assignments in the RHS, we have
$$
\perm\big(\uc{A}([1,\dots,n],[j]^{c})\big)\geq \perm\big(\uc{A}([\phi(1),\dots,\phi(d)],[j]^{c})\big).
$$
To prove the reverse inequality we will need the following results (\ref{result1},\ref{result2}).

\indent {\bf i)} For $\uc{B}\in\Rmax^{m\times \ell}$, $\phi\in\oas(\uc{B})$ and $i\in\{1,\dots,m\}$, either $\phi$ does not assign now $i$, or it assigns row $i$ to some column $j\in\{1,\dots,\ell\}$. This yields
\begin{align}
\label{result1}
\perm(\uc{B})&=\perm\big(\uc{B}([i]^{c},[1,\dots,k])\big) \\ &\oplus \max_{j=1}^{\ell}\Big(\lc{b}_{ij}+\perm\big(\uc{B}([i]^{c},[j]^{c})\big)\Big). \nonumber
\end{align}
The expression on the first line of the RHS of \eqref{result1} is the maximum over all assignments that do not assign row $i$ and the expression on the second line is the maximum over $j\in\{1,\dots,\ell\}$, of the maximum over all assignments that assign row $i$ to column $j$.  $\square$

\indent {\bf ii)} For $\uc{B}\in\Rmax^{m\times \ell}$, $\phi\in\oas(\uc{B})$, we have
\begin{equation}\label{result2}
\perm\big(\uc{B}([\phi(j_1),\dots,\phi(j_k)]^c,[j_1,\dots,j_k]^c)\big)=\sum_{t \neq j_1,\dots,j_k}\lc{b}_{\phi(t)t}.
\end{equation}
First note that by restricting $\phi$ to the rows and columns of the submatrix on the LHS of \eqref{result2} we obtain an assignment with weight equal to the expression on the RHS, so that $LHS\geq RHS$. Now suppose that $LHS>RHS$, then there exists an assignment $\phi'$ of the submatrix with weight strictly greater than that of $\phi$. But we can extend $\phi'$ to an assignment of the full matrix $\uc{B}$ by assigning $j_t$ to $\phi(j_t)$ for $t=1,\dots,k$. This results in an assignment of $\uc{B}$ with weight strictly greater than that of $\phi$, which is a contradiction. $\square$

We construct a sequence $j_1,\dots,j_k$, as follows. Set $j_1=j$, as in the statement of the Lemma, then from \eqref{result1} we have
\begin{align*}
\perm &\big(\uc{A}([1,\dots,n],[j_1]^{c})\big)= \perm\big(\uc{A}([\phi(j_1)]^{c},[j_1 ]^{c})\big) \\ & \oplus\max_{t\neq j_1}\Big(\lc{a}_{\phi(j_1) t}+ \perm\big(\uc{A}([\phi(j_1)]^{c},[j_1,t]^{c})\big)\Big).
\end{align*}
If the expression in the first line of the RHS attains the maximum we stop, otherwise we set $j_2$ to be a value of $t$ that attains the maximum in the second line of the RHS. After $k-1$ steps we have $j_1,\dots,j_k$, a sequence of distinct elements of $\{1,\dots,d\}$. From \eqref{result1} we have 
\begin{align*}
\perm &\big(\uc{A}([\phi(j_1),\dots,\phi(j_{k-1})]^{c},[j_1,\dots,j_k]^{c})\big)= \perm\big(\uc{A}([\phi(j_1),\dots,\phi(j_{k})]^{c},[j_1,\dots,j_k]^{c})\big) \\ & \oplus \max_{t\neq j_1,\dots,j_k}\Big(\lc{a}_{\phi(j_k) t}+ \perm\big(\uc{A}([\phi(j_1),\dots,\phi(j_{k-1})]^{c},[j_1,\dots,j_k,t]^{c})\big)\Big).
\end{align*}
If the expression in the first line of the RHS attains the maximum we stop, otherwise we set $j_{k+1}$ to be a value of $t$ that attains the maximum in the second line. Continuing in this way we either generate a sequence of length $d$, in which case
\begin{equation}\label{stop1}
\perm\big(\uc{A}([1,\dots,n],[j_1]^c)\big)=\sum_{t=1}^{d-1} \lc{a}_{\phi(j_{t}) j_{t+1}},
\end{equation}
or we stop after $k<d$ steps, in which case
\begin{equation}\label{stop2}
\perm\big(\uc{A}([1,\dots,n],[j_1]^c)\big)=\sum_{t=1}^{k-1} \lc{a}_{\phi(j_{t}) j_{t+1}}+\perm\big(\uc{A}([\phi(j_1),\dots,\phi(j_k)]^c,[j_1,\dots,j_k]^c)\big).
\end{equation}
The expression in the RHS of \eqref{stop1} is the weight of an assignment of $\uc{A}$ that only assigns the rows $\{\phi(1),\dots,\phi(d-1)\}$ and does not assign the column $j_1=j$, so that 
$$
\perm\big(\uc{A}([1,\dots,n],[j_1]^c)\big)=\sum_{t=1}^{d} \lc{a}_{\phi(j_{t}) j_{t+1}}\leq \perm\big(\uc{A}([\phi(1),\dots,\phi(d)],[j_{1}]^{c})\big).
$$
Applying result \eqref{result2} to \eqref{stop2} yields 
$$
\perm\big(\uc{A}([1,\dots,n],[j_1]^c)\big) =\sum_{t=1}^{k-1} \lc{a}_{\phi(j_{t}) j_{t+1}} + \sum_{t=k+1}^d \lc{a}_{\phi(j_t)j_t}.
$$
The expression in the RHS is the weight of an assignment of $\uc{A}$ that only assigns the rows $\{\phi(1),\dots,\phi(d)\}/\phi(k)$ and does not assign the column $j$, so that 
\begin{align*}
\perm\big(\uc{A}([1,\dots,n],[j_1]^c)\big) &=\sum_{t=1}^{k-1} \lc{a}_{\phi(j_{t}) j_{t+1}} + \sum_{t=k+1}^d \lc{a}_{\phi(j_t)j_t} \\ &\leq \perm\big(\uc{A}([\phi(1),\dots,\phi(d)],[j]^{c})\big).
\end{align*}

\end{proof}

\begin{theorem}\label{mpthm}Let $\uc{A}\in\Rmax\nbd$ and without loss of generality assume that $(1,2,\dots,d)\in\oas(\uc{A})$ and set $\uc{M}=\uc{A}([1,\dots,d],[1,\dots,d])$, then
$$
\lc{p}_{i}(\uc{A})=\left\{\begin{array}{cc} 0 & \hbox{for $i=1,\dots,d$}, \\ 2\big(\uc{A}\otimes  \uc{M}^{\otimes -1}\otimes \underline{0}\big)_{i} & \hbox{for $i=d+1,\dots,n$}. \end{array}\right.
$$

\end{theorem}
\begin{proof}
 First note that if $(1,2,\dots,d)$ is not an optimal assignment of $\uc{A}$ then we can permute the rows of $\uc{A}$ so that it is. Permuting the rows will permute the statistical leverage scores in the same way. For the assigned rows $i=1,\dots,d$, we have $\lc{p}_{i}(\uc{A})=0$, which matches the definition of $\lc{p}(\uc{A})$ given in \eqref{mpslsdef2}. For the remaining rows $i=d+1,\dots,n,$ recall that
\begin{equation}\label{def2}
\lc{p}_{i}(\uc{A})=2\big(\perm(\uc{A},i)-\perm(\uc{A})\big),
\end{equation}
where that the $i$-obligated permanent $\perm(\uc{A},i)$ is the weight of the maximally weighted assignment that assigns row $i$. Such an assignment must assign row $i$ to some column $k\in\{1,\dots,d\}$. Taking the maximum over the column assigned to $i$ yields
\begin{equation}\label{mm1}
\perm(\uc{A},i)=\max_{k=1}^{d}\Big(\lc{a}_{ik}+\perm\big(\uc{A}([k]^c,[i]^c\big)\big)\Big).
\end{equation}
Since $\phi=(1,\dots,n)$ is an optimal assignment of $\uc{A}$ it is also an optimal assignment of $\uc{A}([i]^c,[1,\dots,d])$, so from Lemma \ref{trickylemma} we have
$$
\perm\big(\uc{A}([i]^c,[k]^c\big)\big)=\perm(\uc{A}([1,\dots,d],[k]^c )\big)=\perm(\uc{M}([1,\dots,d],[k]^c)\big).
$$
The expression on the RHS is the permanent of a matrix with $d$ rows and $d-1$ columns, so any optimal assignment will have to leave one row $j\in\{1,\dots,d\}$ unassigned. Taking the maximum over the unassigned row gives
$$
\perm\big(\uc{M}([1,\dots,d],[k])\big)=\max_{j=1}^{d}\Big(\perm\big(\uc{M}([j]^c,[k]^c)\big)\Big).
$$
Substituting this expression back into \eqref{mm1} then \eqref{def2} and also using the fact that $\perm(\uc{A})=\perm(\uc{M})$, we have
\begin{align*}
\lc{p}_{i}(\uc{A})&=2\max_{k=1}^{d}\left(\lc{a}_{ik}+\max_{j=1}^{d}\Big(\perm\big(\uc{M}([j]^c,[k]^c)\big)\Big)\right)-\perm(\uc{M}) \\ &= \max_{j=1}^{d} 2\big(\uc{A}\otimes  \uc{M}^{\otimes -1}\big)_{ij} \\ &=\big(\uc{A}\otimes  \uc{M}^{\otimes -1}\otimes{0}\big)_{i}.
\end{align*}
\end{proof}

\begin{theorem}\label{lemthm}Let $\tilde{A}\in\P^{n\times d}$ and suppose that $\uc{L}(\tilde{A})\in\uc{G}(n,d)$, without loss of generality assume that $(1,2,\dots,d)\in\oas\big(V(\tilde{A})\big)$ and set $\tilde{M}=\tilde{A}([1,\dots,d],[1,\dots,d])$, then

\indent {\bf a)} For $j=1,\dots,d$
$$
V(\tilde{A}\tilde{M}^{-1})_{ij}=\left\{\begin{array}{cc} \delta_{ij} & \hbox{for $i=1,\dots,d$}, \\ \big(V(\tilde{A})\otimes  V(\tilde{M})^{\otimes -1}\big)_{ij} & \hbox{for $i=d+1,\dots,n$}. \end{array}\right.
$$

%


\indent {\bf b)}
$$
\lim_{z\rightarrow\infty}\frac{\log p_{i}\big(\tilde{A}(z)\big)}{\log |z|}=\lc{p}_{i}(\uc{A}), \quad \hbox{for $i=1,\dots,n$}.
$$

\end{theorem}
\begin{proof} {\bf a)} From Lemma \ref{submatrix} and Lemma \ref{detperm} we have that $\tilde{M}$ is invertible. Therefore $\tilde{A}\tilde{M}^{-1}$ is a matrix whose first $d$ rows form the $d\times d$ identity matrix and whose remaining entries can be expressed using Cramer's rule
$$
(\tilde{A}\tilde{M}^{-1})_{ij}=(\tilde{A}_{i\cdot}\tilde{M}^{-1})_{j}=\det\big(\tilde{M}(\tilde{A}_{i\cdot},j)\big)/\det(\tilde{M}), \quad \hbox{for $i=d+1,\dots,n$, $j=1,\dots,d$},
$$
where $\tilde{M}(\tilde{A}_{i\cdot},j)$ is the matrix obtained by replacing the $j$th row of $\tilde{M}$ with the $i$th row of $\tilde{A}$. Using Lemma \ref{submatrix} and Lemma \ref{detperm} we have 
$$
V(\tilde{A}\tilde{M}^{-1})_{ij}=\perm\Big( V\big(\tilde{M}(\tilde{A}_{i\cdot},j)\big)\Big)-\perm\big(V(\tilde{M})\big).
$$
We can expand the permanent of a $d\times d$ max-plus matrix along a row just like in the classical case of a determinant. Expanding along the $j$th row of $V\big(\tilde{M}(\tilde{A}_{i\cdot},j)\big)$ results in
$$
\perm\Big( V\big(\tilde{M}(\tilde{A}_{i\cdot},j)\big)\Big)=\max_{k=1}^{d}\perm\Big(V\big( \tilde{M}([1,\dots,d]/k,[1,\dots,d]/j)\big)\Big)+V(\tilde{A}_{ik}),
$$
which yields
$$
V(\tilde{A}\tilde{M}^{-1})_{ij}=\big(V(\tilde{A})\otimes  V(\tilde{M})^{\otimes -1}\big)_{ij}, \quad \hbox{for $i=d+1,\dots,n$ and $j=1,\dots,d$.} ~\square
$$

\indent {\bf b)} From Lemma \ref{submatrix} and Lemma \ref{detperm} we have that $\tilde{M}$ is invertible so that the statistical leverage scores of $\tilde{A}$ and $\tilde{R}=\tilde{A}\tilde{M}^{-1}$ are equal. Note that since the first $d$ rows of $\tilde{R}$ form the identity matrix we have
\begin{equation}\label{pp1}
 \max_{j=1}^{d}|x_{j}|\leq \|\tilde{R}(z)x \|_{2}, \quad \hbox{for all $z\in \C$, $x\in\C^d$}.
\end{equation}
Also note that 
\begin{equation}\label{pp2}
|(\tilde{R}(z)x)_i  |\leq  d \big(\max_{j=1}^{d}|\tilde{r}_{ij}(z)|\big) \max_{j=1}^{d}|x_{j}|, \quad \hbox{for all $z\in \C$, $x\in\C^d$, $i=1,\dots,n$}.
\end{equation}
Therefore
\begin{align*}
\nonumber \lim_{z\rightarrow 0}\frac{-\log p_{i}\big(\tilde{A}(z)\big)}{\log |z|} &=\lim_{z\rightarrow 0}\frac{-1}{\log |z|}\log\left(\max_{x\in\C^d}\frac{|(\tilde{R}(z)x)_{i}|}{\|\tilde{R}(z)x\|_{2}}\right)^{2} \\ \nonumber &=\lim_{z\rightarrow 0}\frac{-2}{\log |z|} \max_{x\in\C^d}\left({\log |\big(\tilde{R}(z)x\big)_{i} |}-{\log\|\tilde{R}(z)x \|_{2}}\right) \\ \nonumber & \leq \lim_{z\rightarrow 0}\frac{-2}{\log |z|} \max_{x\in\C^d}\left({\log \Big(d \big(\max_{j=1}^{d}|\tilde{r}_{ij}(z)|\big) \max_{j=1}^{d}|x_{j}|\Big)}-{\log\big(\max_{j=1}^{d}|x_{j}|\big)}\right) \\  &= \lim_{z\rightarrow 0}\frac{-2\max_{j=1}^{d}|\tilde{r}_{ij}(z)| }{\log |z|}=2\max_{j=1}^{d}V(\tilde{r}_{ij})=\lc{p}(\uc{A})_{i},
\end{align*}
where we used the results of part {\bf a} and Theorem~\ref{mpthm} in the last line. Next we will prove the reverse inequality. Setting $x=\underline{e}_{\ell}$, where $\ell=\arg\max_{j=1}^d V(\tilde{r}_{ij})$, yields
\begin{align}\label{ineq2}
\nonumber \lim_{z\rightarrow 0}\frac{-\log p_{i}\big(\tilde{A}(z)\big)}{\log |z|} &\nonumber \geq\lim_{z\rightarrow 0}\frac{-1}{\log |z|}\log\left(\frac{|(\tilde{R}(z)\underline{e}_{\ell})_{i}|}{\|\tilde{R}(z)\underline{e}_{\ell}\|_{2}}\right)^{2} \\ &\nonumber =\lim_{z\rightarrow 0}\frac{-2\log |\big(\tilde{R}(z)\underline{e}_{\ell}\big)_{i} |}{\log|z|}-\frac{-2\log\|\tilde{R}(z)\underline{e}_{\ell} \|_2 }{\log|z|} \\ &= 2V(\tilde{r}_{i\ell})-\lim_{z\rightarrow 0}\frac{-2\log\|\tilde{R}(z)\underline{e}_{\ell} \|_2 }{\log|z|}.
\end{align}
Now
$$
\|\tilde{R}(z)\underline{e}_{\ell}\|_{2}\leq \sqrt{n}\max_{k=1}^{n}|\tilde{r}(z)_{kl}|  , \quad \hbox{for all $z\in \C$, $x\in\C^d$},
$$
so that
\begin{equation}\label{xj}
\lim_{z\rightarrow 0}\frac{-\log\|\tilde{R}(z)\underline{e}_{\ell} \|_2 }{\log|z|}\leq \lim_{z\rightarrow 0} \max_{k=1}^{n}\frac{-\log |\tilde{r}(z)_{kl}|}{\log |z|}=  \max_{k=1}^{n}V(\tilde{r}_{k\ell})=0,
\end{equation}
where we used the fact that $V(\tilde{b}_{k\ell})\leq \lc{p}(\uc{A})_{k}\leq 0$ and that $V(\tilde{b}_{\ell\ell})=0$. Finally substituting \eqref{xj} into \eqref{ineq2} and noting from the results of part {\bf a} and Theorem~\ref{mpthm}, that $2V(\tilde{b}_{i\ell})=\lc{p}(\uc{A})_{i}$, results in
$$
 \lim_{z\rightarrow 0}\frac{-\log p_{i}\big(\tilde{A}(z)\big)}{\log |z|}\geq \lc{p}(\uc{A}_{i}).
 $$
\end{proof}

\begin{example}
Consider
$$
\tilde{A}=\left[\begin{array}{cc} z^{-3} & z^{-3} \\ 1 & z^{-2} \\ z^{-1} & 1 \end{array}\right], \quad \uc{A}=V(\uc{A})=\left[\begin{array}{cc} 3 & 3 \\ 0 & 2 \\ 1 & 0 \end{array}\right], \quad L(\tilde{A})=\left[\begin{array}{cc} 1 & 1 \\ 1 & 1 \\ 1 & 1 \end{array}\right].
$$
We have chosen $\tilde{A}$ so that $\tilde{A}(0.1)=A$, where $A$ is the fixed complex matrix of Example \ref{goodegg1}. Note that we do not have $L(\tilde{A})\in\uc{G}(3,2)$, so that this matrix does not actually satisfy the condition of Theorem~\ref{lemthm}. Although we will see that this is not a problem for this particular example. To obtain matrices which satisfy the condition we randomly generate two more matrices of Puiseux series with the same valuation as $\tilde{A}$ but with independent Gaussian distributed leading order term coefficients. We plot the behavior of the statistical leverage scores of $\tilde{A}(z)$ as a function of $z$ as well as the scores of two of the randomly generated matrices. See Figure \ref{limplot}. Observe that the limits of the randomly signed matrices converge to $\lc{p}(\uc{A})=[0,0,-2]$ as required. In this example the condition $L(\tilde{A})\in\uc{G}(3,2)$ is not necessary and the scores of the matrix $\tilde{A}(z)$ also converge to the max-plus scores. We have highlighted the scores of $\uc{A}(10^{-k})$ for $k=1,2,3$. Note that  $\uc{A}(10^{-k})$ is equal to the elementwise $k$th power of $A$.

 \begin{figure}
\begin{center}
\subfigure[]{\includegraphics[scale=0.4]{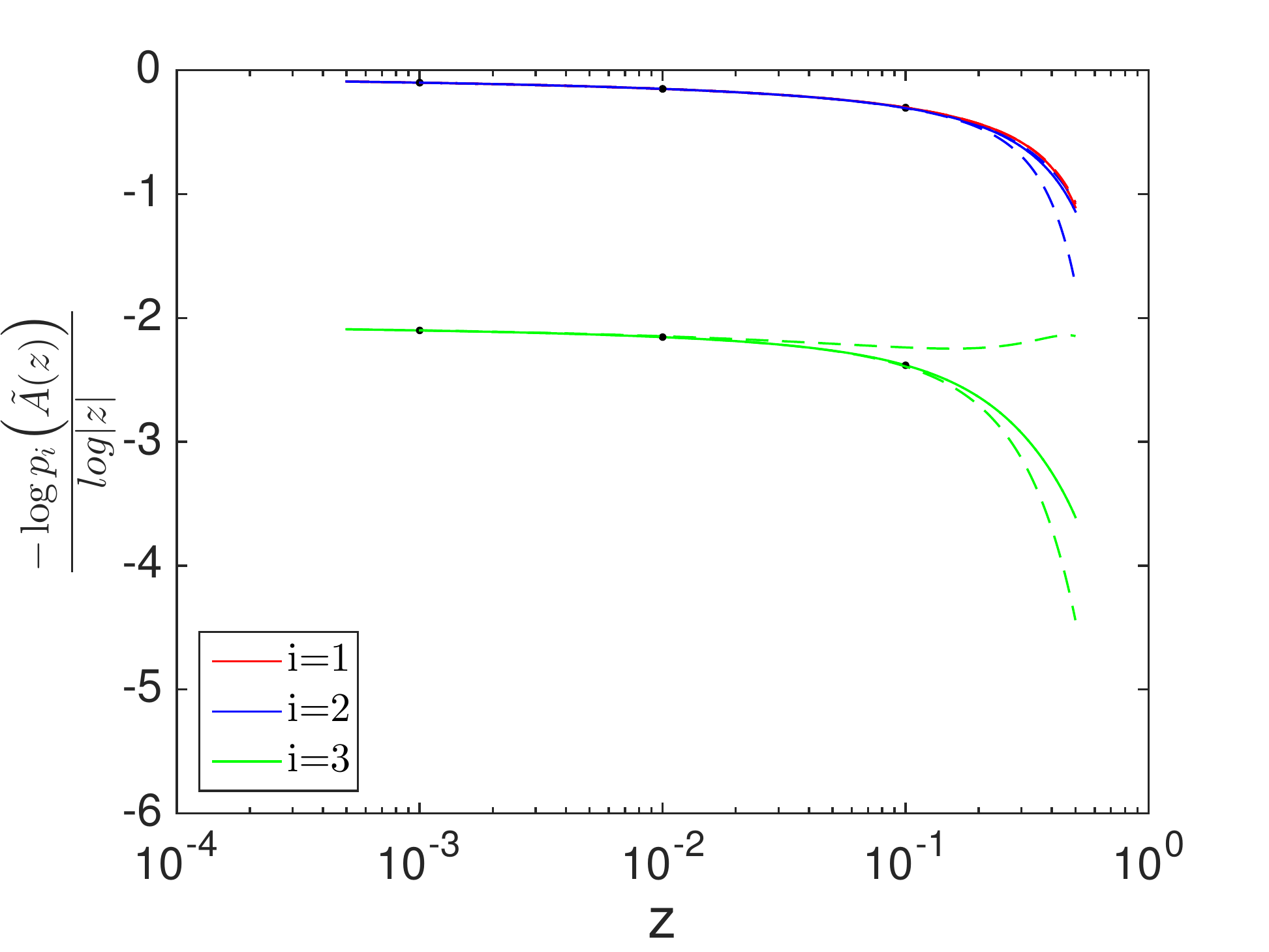}}\hspace{.4cm}
\subfigure[]{\includegraphics[scale=0.4]{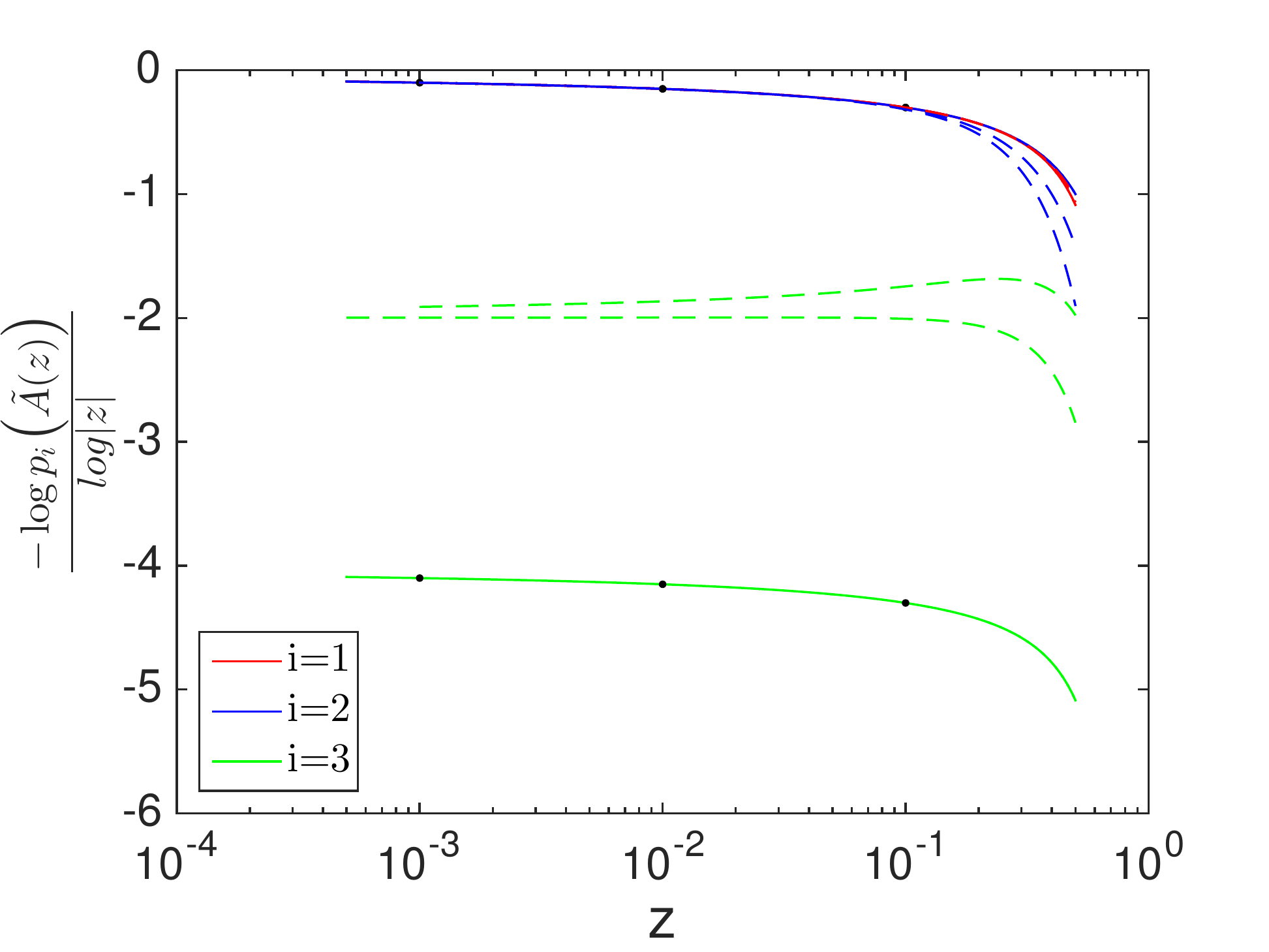}}\hspace{.4cm}
\caption{Convergence of $\frac{-\log p_{i}\big(\tilde{A}(z)\big)}{\log |z|}$, with scores of randomly generated matrices marked by dashed lines. a) Example \ref{goodegg1}, b) Example \ref{badegg1}.}
\label{limplot}
\end{center}
\end{figure}

\end{example}

\begin{example}
Consider
$$
\tilde{A}=\left[\begin{array}{cc} z^{-3} & z^{-3} \\ 1 & z^{-2} \\ z^{-1} & z^{-1} \end{array}\right], \quad \uc{A}=V(\uc{A})=\left[\begin{array}{cc} 3 & 3 \\ 0 & 2 \\ 1 & 1 \end{array}\right], \quad L(\tilde{A})=\left[\begin{array}{cc} 1 & 1 \\ 1 & 1 \\ 1 & 1 \end{array}\right].
$$
Now we have chosen $\tilde{A}$ so that $\tilde{A}(0.1)=A$, where $A$ is the fixed complex matrix of Example \ref{badegg1}. We make the same plot as before. See Figure \ref{limplot}.  As before the scores of the randomly signed matrices converge to the max-plus limits as required by  Theorem~\ref{lemthm}. However the $3$rd score of $\tilde{A}(z)$ converges to a different limit. In this second example the condition $L(\uc{A})\in\uc{G}(3,2)$ is still not strictly necessary, but the condition  $l_{11}l_{32}-l_{22}l_{31}\neq 0$ is, and it is not not satisfied by $\tilde{A}$, which is why its limits disagree with the max-plus statistical leverage scores. 

\end{example}

\section{Max-plus statistical leverage score algorithm}\label{asection}

Theorem \ref{mpthm} shows us how to calculate the max-plus statistical leverage scores $\lc{p}(\uc{A})$ for $\uc{A}\in\Rmax\nbd$. First of all we compute an optimal assignment $\phi\in\oas(\uc{A})$. Next we set $\uc{M}=\uc{A}([\phi(1),\dots,\phi(d)],[1,\dots,d])$ and compute $\uc{M}^{\otimes -1}$, then
\begin{equation}\label{algp}
\lc{p}_{i}(\uc{A})=\left\{\begin{array}{cc} 0 & \hbox{for $i$ is assigned by $\phi$}, \\ 2\big(\uc{A}\otimes  \uc{M}^{\otimes -1}\otimes \underline{0}\big)_{i} & \hbox{otherwise}. \end{array}\right.
\end{equation}
See Algorithm~\ref{alg:MPSLS}. We treat the computations of the optimal assignment and max-plus inverse separately below. The multiplication on line 4 has cost $\uc{O}(d^2)$. For $i=1,\dots,n$, setting $\lc{p}_{i}(\uc{A})$ has cost $\uc{O}(d)$, so that the total cost of setting $\lc{p}(\uc{A})$ is $\uc{O}(nd)$. Note that each row can be treated independently in parallel and that if $\uc{A}$ is a sparse matrix\footnote{A sparse max-plus matrix is one with many entries equal to minus infinity. If $A\in\C\nbd$ is a conventional sparse matrix then $\log |A|\in\Rmax\nbd$ is a sparse max-plus matrix.} then the total cost of setting $\lc{p}(\uc{A})$ is $\uc{O}(\tau)$, where $\tau$ is the number of finite entries in $\uc{A}$.

To compute an optimal assignment $\phi\in\oas(\uc{A})$ we can use the Hungarian algorithm \cite{Munkres1957}, the Successive Shortest Paths algorithm \cite{orle93} or the Auction algorithm \cite{Bertsekas1989}. Applied directly to $\uc{A}\in\Rmax\nbd$ all of these algorithms have cost $\uc{O}(nd^{2})$. However we can reduce this cost considerably by noting that the optimal assignment of a tall skinny $n\times d$ matrix only depends on the $d$ largest entries in each row. For each row of $\uc{A}$ we select the $d$ largest entries, which we then sort in decreasing order. This results in a sparse matrix with at most $d$ entries per column and a known sorting order for each column. We then pass this matrix to the Successive Shortest Paths algorithm, which is able to compute the optimal assignment with cost $\uc{O}(d^3)$. An efficient implementation which exploits the fact that the $d$ largest entries in each column have been sorted is essential to achieve this lower cost. To select the $d$ largest entries in each column we use Quickselect \cite{Hoare1961}. Like Quicksort this algorithm has poor worst case cost but good average case cost. For Quickselect to find the $d$ largest entries in a column of length $d$ has worst case cost $\uc{O}(n^2)$ but average case cost $\uc{O}(n)$. We then sort the $d$ largest entries from each column using Quicksort with average case cost $\uc{O}\big(d\log(d)\big)$. Using this approach the total average case cost of computing the optimal assignment $\phi$ is $\uc{O}(nd+d^3)$. Clearly each column can be treated independently in parallel and if $\uc{A}$ is a sparse matrix then the total cost is $\uc{O}(\tau+d^3)$.

To compute the max-plus inverse $\uc{M}^{\otimes -1}$ we adapt the approach taken in (\cite{hoti16}, Appendix A), where the authors present an algorithm for computing max-plus LU factors. For $\uc{M}\in\Rmax^{d\times d}$ and $\pi\in\oas(\uc{M})$, let $\uc{P}_{\pi}\in\Rmax^{d\times d}$ be the max-plus permutation matrix with 
\begin{equation}
(\uc{P}_{\pi})_{ij}=\left\{\begin{array}{cc} 0 & \hbox{if $i=\pi(j)$} \\ -\infty & \hbox{otherwise}. \end{array}\right.
\end{equation}
There exists max-plus diagonal matrices\footnote{A max-plus diagonal matrix is one whose off diagonal entries are all equal to minus infinity.} such that 
\begin{equation}
\uc{H}=\uc{P}_{\pi}\otimes \uc{D}_{1}\otimes \uc{M}\otimes \uc{D}_{2},
\end{equation}
satisfies $\lc{h}_{ij}\leq 0$ and $\lc{h}_{ii}=0$ for all $i,j=1,\dots,d$. We say that $\uc{H}$ is a \emph{Hungarian scailng} of $\uc{M}$. The coefficients of the diagonal scaling matrices are given by the dual variables in the LPP form of the optimal assignment problem. So that primal dual algorithms for computing the optimal assignment of a matrix, like those listed above, will also produce these scaling coefficients a byproduct. Ordinarily we would need to apply one of these algorithms to $\uc{M}$ with worst case cost $\uc{O}(d^3)$, but in this setting we can use the results from the previous computation applied to $\uc{A}$.

To compute the max-plus inverse of $\uc{M}$ we use the formula
\begin{equation}
\uc{M}^{\otimes -1}=\uc{D}_{1}\otimes \uc{H}^{\otimes -1}\otimes {D}_{2},
\end{equation}
where the entries in the inverse of the Hungarian matrix $\uc{H}$ can be calcualted as follows. Let $G(\uc{H})$ be the graph with vertices $\{1,\dots,d\}$ and an edge $i\mapsto j$ with weight $\lc{h}_{ij}$ whenever $\lc{h}_{ij}\neq-\infty$, then
\begin{equation}
\big(\uc{H}^{\otimes -1}\big)_{ij}=\hbox{weight of the maximally weighted path through $G(\uc{H})$ from $i$ to $j$}.
\end{equation}
Each row of $\uc{H}^{\otimes -1}$ can be computed by independently using Dijkras algorithm, with a total worst case cost of $\uc{O}(d^3)$ for a dense matrix.

The total average case cost of Algorithm~\ref{alg:MPSLS} is therefore $\uc{O}(nd+d^3)$ or $\uc{O}(\tau+d^3)$ in the sparse case.


\begin{algorithm}

\caption{ \label{alg:MPSLS}
Given a max-plus matrix $\uc{A}\in\Rmax^{n\times d}$, compute $\lc{p}(\uc{A})$. }

\medskip

\begin{algorithmic}[1]
\State compute an optimal assignment $\phi\in\oas(\uc{A})$
\State set $\uc{M}=\uc{A}([\phi(1),\dots,\phi(d)],[1,\dots,d])$
\State compute $\uc{M}^{\otimes -1}$
\State set $\lc{x}=\uc{M}^{\otimes -1}\otimes \underline{0}$
\For {$i=1,\dots,n$}
\If {$i$ assigned by $\phi$} 
\State set $\lc{p}_{i}(\uc{A})=0$
\Else  
\State set $\lc{p}_{i}(\uc{A})=(\uc{A}\otimes \lc{x})_{i}$
\EndIf
\EndFor

\end{algorithmic}

\end{algorithm}

\section*{Conclusion}

We presented a max-plus algebraic analogue for statistical leverage scores. Max-plus statistical leverage scores can be used to calculate the exact asymptotic behavior of the conventional statistical leverage  scores of generic matrices of Puiseux series and also provide a novel way to approximate the conventional statistical leverage scores of fixed complex matrices. In Section~\ref{nsection} we showed that the max-plus approximation was accurate to within an order of magnitude for a small test set of randomly generated matrices. We also demonstrated that this order of magnitude approximation was sufficiently accurate to be useful in practice, as an importance sampling distribution used inside a sampling based least squares solver. 

A major drawback of the max-plus approximation is the fact that there are certain problem matrices for which it will be very inaccurate. See Example \ref{badegg1}. However, as demonstrated in Section~\ref{maxplussection}, these `nasty' matrices where the max-plus approximation is very inaccurate form a set of small measure. Thus we expect that the max-plus approximation will be able to capture the order of magnitude of all of the statistical leverage scores of a randomly generated matrix, or a deterministic matrix which `looks like random matrix'. This should cover many object-feature matrices arising in machine learning applications.

%

\bibliographystyle{plain}

\bibliography{jlhbib}

\end{document}